\documentclass{article}

\usepackage[preprint,nonatbib]{neurips_2022}
\usepackage[utf8]{inputenc} 
\usepackage[T1]{fontenc}    
\usepackage{hyperref}       
\usepackage{url}            
\usepackage{booktabs}       
\usepackage{amsfonts}       
\usepackage{nicefrac}       
\usepackage{microtype}      
\usepackage{xcolor}         
\usepackage{xfrac}          

\usepackage{amsmath}
\usepackage{amsthm}
\usepackage{amssymb}
\newtheorem{theorem}{Theorem}
\newtheorem{proposition}{Proposition}
\newtheorem{lemma}{Lemma}
\newtheorem{definition}{Definition}
\newtheorem{corollary}{Corollary}
\newtheorem{remark}{Remark}

\def\DiffMc{\text{Diff}_c(\mathcal{M})}
\def\DiffM{\text{Diff}(\mathcal{M})}
\def\M{\mathcal{M}}

\def\Cc{\mathcal{C}^\infty_c(\mathcal{M})}
\def\CcV{\mathcal{C}^\infty_c(\mathcal{M},T\mathcal{M})}
\def\Lp{L^p_\omega(\mathcal{M,\mathbb{R}})}
\def\LpV{L^p_\omega(\mathcal{M},T\mathcal{M})}
\def\O{\mathcal{O}_1}
\def\Op{\dot{\mathcal{O}_1}}

\makeatletter
 \let\Ginclude@graphics\@org@Ginclude@graphics
\makeatother
  \usepackage[all]{xy}

\title{On Non-Linear operators for Geometric Deep Learning}

\author{%
  Gr\'egoire Sergeant-Perthuis\\
  Univ. Artois, UR 2462, Laboratoire de\\
  Math\'ematiques de Lens (LML)\\
  F-62300 Lens, France\\
  \& OURAGAN team, Inria Paris \& IMJ-PRG\\ Paris, France.\\
  \texttt{gregoireserper@gmail.com}
  \And 
  Jakob Maier\\
  INRIA, DI/ENS, PSL\\
  Paris
  \And
  Joan Bruna\\
  Courant Institute of Mathematical Sciences\\
  New York University\\
New York
\And 
Edouard Oyallon\\
 MLIA - Machine Learning \\
 and Information Access\\
Sorbonne Université, CNRS,\\
ISIR, F-75005 Paris, France
}

\usepackage{soul}

\usepackage{comment}
\begin{document}

\maketitle

\begin{abstract}%
  This work studies operators mapping vector and scalar fields defined over a manifold $\M$, and which commute with its group of diffeomorphisms $\DiffM$. We prove that in the case of scalar fields $\Lp$, those operators correspond to point-wise non-linearities, recovering and extending known results on $\mathbb{R}^d$. In the context of Neural Networks defined over $\M$, it indicates that point-wise non-linear operators are the only universal family that commutes with any group of symmetries, and justifies their systematic use in combination with dedicated linear operators commuting with specific symmetries. In the case of vector fields $\LpV$, we show that those operators are solely the scalar multiplication. It indicates that $\DiffM$ is too rich and that there is no universal class of non-linear operators to motivate the design of Neural Networks over the symmetries of $\M$.
\end{abstract}

\section{Introduction}

  Given a physical domain $\M$ and measurements $f:\M \to \mathcal{Y}$ observed over it, one is often interested in processing \emph{intrinsic} information from $f$, i.e. consistent with the \emph{symmetries} of the domain. {\color{black} Let $M$ denote an operator, it can be seen as a non-linear operator acting on measurements.} In words, if two measurements $f$, $\tilde{f}=g.f$ are related by a symmetry $g$ of the domain, like a rigid motion on an observed molecular compound, we would like our processed data $M(f)$ and $M(\tilde{f})$ to be related by the same symmetry --- thus that $M(g.f)=g.M(f)$ or equivalently that $M$ commutes with the symmetry transformation of the domain. 
  The study of operators that satisfy such symmetry constraints has played a long and central role in the history of physics and mathematics, motivated by the inherent symmetries of physical laws. More recently, such importance  has also extended to the design of machine learning systems, where symmetries improve the sample complexity ~\cite{mei2021learning,bietti2021sample}. For instance, Convolutional Neural Networks build translation symmetry, whereas Graph Neural Networks build permutation symmetry, amongst other examples coined under the `Geometric Deep Learning' umbrella \cite{bronstein2017geometric,bronstein2021geometric}.

 Lie groups of transformations are of particular interest, because there exists a precise and systematic framework to build such intrinsic operators.  Indeed, for a locally compact group $G$, it is possible to define a Haar measure  which is invariant to the action of $G$~\cite{bandt1983metric}; then a simple filtering along the orbit of $G$ allows to define a class of \emph{linear} operators that commute with the group action. Examples of locally compact groups are given by specific Lie groups acting on $\mathbb{R}^d$, such as the translations or the rotations $O_d(\mathbb{R})$. Often these  Lie groups $G$ only act on a manifold $\M$, and one tries to average along the orbit induced by $G$.  Note that it is possible, beyond invariance, to linearize more complex groups of variability like diffeomorphisms $\DiffM$~\cite{bruna2013invariant}.
 
 While the description of such linear intrinsic structures is of central mathematical importance and forms the basis of Representation theory~\cite{ping2002group}, in itself is not sufficient to bear fruit in the context of Representation \emph{learning} using Neural Networks~\cite{cohen2014learning}. Indeed, linear operators do not have the capacity to extract rich information needed to solve challenging high-dimensional learning problems. It is therefore necessary to extend the systematic construction and classification of intrinsic operators to the non-linear case.

With that purpose in mind, our work aims at studying the class of (\textit{non-linear}) operators $M$ which commute with the action of the  group $\DiffM$, the diffeomorphisms over $\M$. This approach will lead to a natural class of non-linear intrinsic operators. Indeed, any group $G$ of symmetries is, by definition, a subgroup of $\DiffM$, and thus commutes with such $M$~\cite{mallat2016understanding}. Consequently, obtaining a non-linear invariant to a symmetry group $G$ could be done by using a cascade of interlacing non-linear operators which commute with $\DiffM$ and linear operators which commute with $G$. \\
A notable example of linear operators that are covariant to the Lie group of translations is a given by the convolutions along the orbit of the group. These can be  constructed thanks to the canonical Haar measure~\cite{sugiura1990unitary}. However, such an approach fails for infinite dimensional groups, like our object of interest: contrary to Lie groups, $\DiffM$ is not locally compact and it is thus not possible to define a Haar measure on this group.

Our first contribution is to demonstrate that the \textit{non-linear} operators which act on vector fields (elements of $\LpV$) and which commute with the group of diffeomorphisms, are actually just scalar multiplications. This implies that $\DiffM$ is too rich to obtain non-trivial operators. Our second contribution is to demonstrate that \textit{non-linear} operators acting on signals in $\Lp$ are pointwise non-linearities. This fills a gap in the results of \cite{bruna2013invariant}, and \emph{a fortiori} justifies the use of point-wise non-linearities in geometric Deep Learning~\cite{bronstein2021geometric}.

{\color{black} Let us remark that the study of equivariant operators that take as input vector fields is motivated by the use of Neural Networks in physics, in particular for dynamical systems such as fluid dynamics \cite{doi:10.1146/annurev-fluid-010719-060214}. For example, one subject of interest in hydrodynamics is how a vector field of velocities evolves; the time evolution of such field is described by a partial differential equation (PDE), the Navier-Stokes equations, in which Neural Networks found recent applications and it is more generally the case of other PDE \cite{RAISSI2019686}.}

Our paper is structured as follows:  Sec. \ref{motivation} introduces the necessary formalism, 
 that we use through this paper: in particular, we formally define the action of diffeomorphism. Then, we state and discuss our theorems in Sec. \ref{theoremstate} and sketch their proofs in Sec. \ref{sketch}. Rigorous proofs of each statement can be found in the Appendix.

\section{Problem Setup\label{motivation}}

\subsection{Related work and motivation}

In this section, we discuss the notion of intrinsic operators, invariant and covariant non-linear operators and linear representation over standard symetry groups. Then, we formally state our objective.

\paragraph{Intrinsic Operators}
As discussed above, in this work we are interested in \emph{intrinsic} operators $M:L^p(\mathcal{M}, E) \to L^p(\mathcal{M}, E)$, where $\mathcal{M}$ is a Riemannian manifold, and $E=\mathbb{R}$ or $E=T\mathcal{M}$, capturing respectively the setting of scalar signals and vector fields over $\mathcal{M}$. {\color{black} $L^p(\mathcal{M}, \mathbb{R})$ is the space of scalar function $f:\mathcal{M}\to \mathbb{R}$ which $p$-th power is integrable, similarly $L^p(\mathcal{M}, T\mathcal{M})$ is the space of sections of the tangent bundle of $\mathcal{M}$ (denoted $T\mathcal{M}$), $f: \mathcal{M}\to T\mathcal{M}$, which norm $\Vert f\Vert : \mathcal{M} \to \mathbb{R}$ is in $L^p(\mathcal{M}, \mathbb{R})$.} Here the notion of `intrinsic' means that $M$ is consistent with an equivalence class induced by a symmetry group $G$ in $L^p(\mathcal{M}, E)$: if $f,\tilde{f} \in L^p(\mathcal{M}, E)$ are related by a transformation $g \in G$ (in which case we write $f =g. \tilde{f}$),  then $M(f) =g. M(\tilde{f})$.
Naturally, a stronger equivalence class imposes a stronger requirement towards $M$, and consequently restrains the complexity of $M$. We  now describe the plausible techniques used to design such operators $M$.

\paragraph{GM-Convolutions}
The notion of $GM$-convolutions \cite{weiler2021coordinate} is an example of linear covariant operators which commute with the reparametrization of a manifold. In practice, this implies that the weights of a $GM$-convolution are shared and the action of $GM$-convolutions is local -- two properties that facilitate implementation and point out the similarity with Lie groups. Another example of symmetry group corresponds to the isometry group of a Riemaniann manifold, whose pushforward preserves the tensor metric. In this case, it is well known that isometries \cite{watson1973manifold} are the only diffeomomorphism which commute with a manifold Laplacian. Thus, any \textit{linear} operators which commute with isometries is stabilized by Laplacian's eigenspaces. However, little is known on the \textit{non-linear} counterpart of the symmetry-covariant operators. In this work, we characterize \textit{non-linear} operators which commute with $\DiffM$. We will see that such operators are intrinsically defined by $\DiffM$ and could be combined with any linear operators covariant with a symmetry group $G$.

\paragraph{Non-linear operators} It has been shown that Convolutional Neural Networks  are dense in the set of \textit{non-linear} covariant operators~\cite{yarotsky2022universal}. The recipe of the corresponding proof is an extension of the proof of the universal approximation theorem~\cite{cybenko1989approximation}. The Scattering Transform~\cite{bruna2013scattering,mallat2012group} is also an example of a well-understood non-linear operator which corresponds to a cascade of complex wavelet transforms followed by a point-wise modulus non-linearity. This representation provably linearizes small deformations.

\paragraph{Compact Lie Groups}
In the context of geometric Machine Learning~\cite{bronstein2017geometric}, there are several relevant notions of equivalence. For instance, we can consider a compact Lie Group $G$ acting on $\M$, and an associated representation in $\mathcal{F}=\{f: \M \to \mathbb{R}\}$: Given $g \in G$ and $f \in \mathcal{F}$, then $g.f(x)\triangleq f(g^{-1}.x)$ for $x \in \M$. We then consider $f \sim \tilde{f}$, related by this group action: $\tilde{f}=g.f$ for some $g \in G$. The operators $M$ which are compatible with such group action are referred as being $G$-equivariant (or covariant to the action of $G$) in the ML literature \cite{cohen2016group,bronstein2021geometric}. Such groups are typically of finite and small dimension, e.g. the Euclidean transformations of $\mathcal{M}=\mathbb{R}^d$, with $d=2$ for computer vision applications, or $d=3$ for computational biology/chemistry applications. In this case, it is possible to characterize all \emph{linear} intrinsic operators $M$ as group convolutions \cite{kondor2018generalization}, leading to a rich family of non-linear intrinsic operators by composing such group convolutions with element-wise non-linear operators, as implemented in modern Neural Networks. We highlight that stability to symetries via non-linear operators finds useful application, in particular for flat manifolds \cite{bruna2013invariant}.

\paragraph{Isometries} 
Riemanian manifolds $\mathcal{M}$ come with a default equivalence class, which is given by isometries. {\color{black} $T_u\mathcal{M}$ denotes the tangent vector space of $\mathcal{M}$ at point $u\in \mathcal{M}$.} If $m_u: T_u\mathcal{M} \times T_u\mathcal{M} \to \mathbb{R}$ denotes the Riemannian metric tensor at point $u\in \mathcal{M}$, a diffeomorphism $\psi: \mathcal{M} \to \mathcal{M}$ is an isometry if $g_u( v, w) = g_{\psi(u)}( d\psi_u(v), d\psi_u(w) )$ for any $u \in \mathcal{M}$ and $v, w \in T_u \mathcal{M}$. In words, isometries are changes of variables that preserve the local distances in the domain. The ensemble of all isometries forms a Lie Group which is locally compact \cite{myers1939group}. In this case, one can also build a rich class of intrinsic operators by following the previously explained `blueprint', namely composing linear intrinsic operators with element-wise non-linearities. As a representative example, the Laplace-Beltrami operator of $\mathcal{M}$ only depends on intrinsic metric properties \cite{watson1973manifold}: as said above, isometries preserve the invariant subspaces of a Laplacian.

\paragraph{Beyond Isometries}
While isometries are the `natural' transformations of the geometric domain, they cannot express high-dimensional sources of variability; indeed, if $\M$ is a $d$-dimensional complete connected Riemannian manifold, its isometry group has dimension at most $d(d+1)/2$ \cite{chen2010riemannian}. This raises the question whether one can characterize intrinsic operators relative to a broader class of transformations. Another class of important symmetries corresponds to the ones which are gauge invariant, i.e. which leads to transformations which preserve the change of parametrization and which are used in \cite{cohen2019gauge,weiler2021coordinate} through the notion of $G$-structure.

In this work, we consider the class of transformations given by $\DiffM$, the diffeomorphisms over $\M$. 
As shown in the Appendix, compactly supported deformations $\psi:\M \to \M$ define bounded linear operators $L_\psi$ acting on $L^p(\M,E) \to L^p(\M,E)$, and constitute a far broader class of transformations than isometries. Our proof is mainly based on the use of compactly supported diffeomorphisms.

Our objective is to characterize the (non-linear) operators $M$ such that
$$\forall \phi \in \DiffM,
L_\phi M=ML_\phi\,.
$$
In other words, we aim to understand continuous operators $M$ that commute with deformations.  We will show that such operators are act locally and that they can be descriped explicitly, with simple formula. The commutation condition is visualized in the following diagram:
$$  \xymatrix{
    \def\commutatif{\ar@{}[rd]|{\circlearrowleft}}
    f \ar[r]^{L_{\phi}} \ar[d]^M \commutatif  & g \ar[d]^M \\
    M f \ar[r]^{L_{\phi}} & M g
  }$$
\subsection{Notations}\label{notations}
We will now formally introduce the mathematical objects of interest in this document. Let $(\mathcal{M},g)$ be an orientable, connected, Riemannian manifold, of finite dimension $d\in \mathbb{N}^*${\color{black}. Let $T\mathcal{M}$ denote the tangent bundle of $\mathcal{M}$, i.e. the union of tangent spaces at points $u\in \mathcal{M}$. $T^*\mathcal{M}$ is the cotangent bundle of $\mathcal{M}$.} {\color{black} $g\in \Gamma( T^*M \otimes T^* M)$ is a section of symmetric definite positive bilinear forms on the tangent bundle of $M$. It is common to denote $\Gamma B$ the collection of sections of a bundle $B$; $\bigwedge^n T^*M$ for $n\leq d$ is the bundle of $n$-linear alternated forms of $\mathcal{M}$, and $\Gamma(\bigwedge^n T^*M)$ is the space of section of this vector bundle over $\mathcal{M}$.}

{\color{black} For $A\subseteq \mathcal{M}$, we denote $\overline{A}$ its closure; $1_A$ is the indicator function of $A$, i.e. which takes value $1$ if $x\in A$ and $0$ otherwise. $\mathcal{B}(u,r)$ denotes the ball of radius $r$ around $u\in \mathcal{M}$. Any two vectors $v,v_1\in V$ in a pre-Hilbert space (with a scalar product $\langle, \rangle$) are orthogonal, denoted $v\perp v_1$, when $\langle v,v_1\rangle =0$.}

Fix $p\in [1,+\infty[$. {\color{black} Any volume form $\omega\in \Gamma( \bigwedge^d T^*M)$ defines a (positive) measure on the orientable Riemannian manifold $\mathcal{M}$; the total volume of $\mathcal{M}$ is $\omega(\mathcal{M}):= \int_{\mathcal{M}} 1 d\omega$.} Let us define $\LpV$, the space of $L^p$ vector fields, defined as the subspace of measurable functions $f:\M\to T\M$ such that $f({\color{black}u})\in T_{{\color{black}u}}M$ {\color{black}almost everywhere} and

\begin{equation}
\Vert f \Vert_p^p \triangleq \int_{{\color{black}u}\in \mathcal M} g_{{\color{black}u}}(f({\color{black}u}),f({\color{black}u}))^{\frac p2}\,d\omega(x) <+\infty\,.
\end{equation}

We will also consider $\Lp$ the space of measurable scalar functions {\color{black}(fields)} $f:\mathcal M\rightarrow \mathbb{R}$ that fulfill

\begin{equation}
    \Vert f\Vert_p^p\triangleq \int_{{\color{black}u}\in \mathcal M} |f({\color{black}u})|^p\,d\omega({\color{black}u})<+\infty\,.
\end{equation}

We may write $\Vert \cdot \Vert$ instead of $\Vert \cdot \Vert_p$ when there is no ambiguity. For a $C^\infty$ diffeomorphism $\phi\in \DiffM$, we will consider the action of $L_\phi:\LpV\rightarrow \LpV$ which we define for for any $f\in \Lp$ as
\[L_\phi f(u)\triangleq d\phi(  u)^{-1}.f(\phi(u))\,.\]
 Note that this action is contravariant:
$$
L_{\psi \circ \phi}f(u)=d(\psi\circ \phi)^{-1}.f(\psi \circ \phi(u))=L_\phi L_\psi f(u)
$$
For scalar function $f\in \Lp$, we define the action of $\phi$ via
\[L_\phi f(u)\triangleq f(\phi(u))\,.\]

{\color{black} Let $A$ be a measurable set of $\mathcal{M}$ and $f\in L^p(\mathcal{M}, E)$, $f1_A$ is the product of $f$ with $1_A$, i.e. $f1_A$ is equal to $f$ on $A$ and $0$ elsewhere. In what follows we introduce 'constant' fields over an open set, they are denoted $c1_U$ with $U$ an open subset of $\mathcal{M}$. For scalar fields, a 'constant' scalar field $f(u)$ is equal to the same constant $c\in \mathbb{R}$ for any $u\in U$. On the other hand, 'constant' vector fields $f1_U$ are vector fields over $U$ for which there is a chart from $U$ to an open subset of $\mathbb{R}^d$, in which for any $u\in U$ $f(u)$ is equal to a constant vector $c\in \mathbb{R}^d$; in the vector case we say that the vector field $f1_U$ can be straightened.}

This latter operator is also contravariant. If there is no ambiguity, we will use the same notation $L_\phi$, whether we apply it to $\Lp$ or $\LpV$. {\color{black} We might sometimes refer to $\Lp$ or $\LpV$ as $L^p(\mathcal{M}, \mathbb{R})$ or $L^p(\mathcal{M}, T\mathcal{M})$.} Throughout the article we restrict ourselves to $\phi$ such that $L_\phi$ is a bounded operator. Write $\text{supp}(\phi)=\{{\color{black}u}, \phi({\color{black}u})\neq {\color{black}u}\}$ for the support of $\phi$ and say that $\phi$ has a compact support if $\text{supp}(\phi)$ is compact. We denote by $\DiffMc\subset\DiffM$ the set of compactly supported diffeomorphisms. Recall that since $\M$ is second-countable, $\Cc$ is dense in $\Lp$ and $\CcV$ is dense in $\LpV$. Finally, denote by $O_d(\mathbb{R})$ the set of unitary operators on $\mathbb{R}^d$. Throughout the article, we might not write explicitly that equalities hold {\color{black}almost everywhere}, since this is the default in $L^p$ spaces.\\
As mentioned earlier, compactly supported diffeomorphisms lead to continuous operators, which is made rigorous by the following lemma whose proof is in the appendix.

\begin{lemma} \label{compact-support}If $\text{supp}(\phi)$ is compact, then $L_\phi$ is bounded.\end{lemma}

\section{Main theorems}\label{theorems}
In this section we present our main results. We first show that any (non-linear) deformation-equivariant operator acting on scalar fields must be point-wise (Theorem  \ref{main-thm-scalar}), and then establish that any deformation-equivariant operator acting on vector fields corresponds to a multiplication by a scalar (Theorem  \ref{main-thm-vector}).

\subsection{Theorem statements}\label{theoremstate}

Now, we are ready to state our two main theorems:

\begin{theorem}[Scalar case]\label{main-thm-scalar}
Let $\M$ be a connected and orientable manifold of dimension $d\geq 1$. We consider a Lipschitz continuous operator $M:\Lp\rightarrow \Lp$, where $1\leq p< \infty$. Then,
$$\forall\, \phi\in\DiffM:\;  ML_\phi=L_\phi M$$
is equivalent to the existence of a Lipschitz continuous function $\rho:\mathbb{R}\to\mathbb{R}$ that fulfills
$$M[f](m)=\rho(f(m)) \quad\text{ {\color{black}a.e.}}$$
In that case, we have $\rho(0)=0$ if $\omega(\M)=\infty$.
\end{theorem}

\begin{theorem}[Vector  case]\label{main-thm-vector}
Let $\M$ be a connected and orientable manifold of dimension $d\geq 1$.  We consider a continuous operator $M:\LpV\rightarrow \LpV$, where $1\leq p< \infty$. Then,
$$\forall\, \phi\in\DiffM:\;  ML_\phi=L_\phi M$$
is equivalent to the existence of a scalar $\lambda \in \mathbb{R}$ such that
$$\forall f\in  \LpV:\, M[f](m)=\lambda f(m) \quad\text{{\color{black}a.e.}}$$

\end{theorem}
We highlight that our theorems are quite generic in the sense that they apply to the manifolds usually used in applications or theory, $\mathbb{R}^d$ in particular.

\begin{remark}
    The scalar case allows to recover standard operators which are exploited for Deep Neural Networks architectures. However, Theorem \ref{main-thm-vector} indicates that the group of diffeomorphism is too rich to obtain non-trivial non-linear operators.
     \end{remark}
     \begin{remark}
    The case $p=\infty$ leads to different results. For instance, in the scalar case we may consider the operator $Mf(x)=\sup_y |f(y)|$ which fulfills $L_\phi Mf=ML_\phi f$ but is not pointwise. 
    \end{remark}

\begin{remark}The condition ``$\omega(\M)=\infty\, \implies \, \rho(0)=0$'' in Theorem \ref{main-thm-scalar} is necessary, since in the case $\M=\mathbb{R}$, the operator $Mf(x)\triangleq e^{if(x)}$ is not in $\Lp$.
\end{remark}
\begin{remark}
The Lipschitz condition in Theorem \ref{main-thm-scalar} is crucial, otherwise, $Mf(x)=\rho(f(x))$ might not be an operator of $\Lp$. For instance, if $p=2$, $\M=[0,1]$ and $Mf(x)=\sqrt{f(x)}$, we see that in this case, let $f(x)=x$, then $f\in \Lp$ and $Mf\not\in \Lp$
\end{remark}

\begin{remark}
If $M$ is not Lipschitz, we can find an example which is not even continuous. The following example holds in both cases, the scalar case and the vector case. In both cases $f\in L^p(M,\mathbb R)$, the only thing that changes is the action of $L_\phi$ on $f$. $\M=\mathbb{R}$, let for all $f\in L^p(M,\mathbb R)$:
$$Mf(x)=1_{\{z,\lim_{y\to z}f(y)=f(z)\}}(x) f(x).$$
It is a measurable function. Let us show that this $M$ is a  counterexample to the vector case: for any $\phi\in\DiffM$ and $x\in \mathbb{R}$, one has
\begin{align}
ML_\phi f(x)&=1_{\{z,\lim_{y\to z}f(\phi(y))=f(\phi(z))\}}(x)\quad d\phi(x)^{-1}f(\phi(x))\\
&=1_{\{z,\lim_{y\to \phi(z)}f(y)=f(\phi(z))\}}(x)\quad d\phi(x)^{-1}f(\phi(x))\\
&=1_{\{z,\lim_{y\to z}f(y)=f(z)\}}(\phi(x))\quad d\phi(x)^{-1}f(\phi(x))\\
&=L_\phi Mf(x)\,.
\end{align}
However, $M$ is not continuous as changing any function to $0$ on $\mathbb{Q}$ does not change its norm but changes the set where the limits exists. More precisely let $c>0$ be a strictly positive scalar, $M[c]=c$; let  $f=c1[x\notin \mathbb Q]$, $M[f]=0$ as $\{z,\exists\lim_{y\to z}f(\phi(y))\}=\emptyset$. However $c=f$ {\color{black}almost everywhere} but $M[c]\neq M[f]$ therefore $M$ is not continuous.  
\end{remark}

\subsection{Proof Sketch}\label{sketch}

We now describe the main ideas for proving the Theorems \ref{main-thm-scalar} and \ref{main-thm-vector}. The appendix contains complete formal arguments and technical lemmata which we omit here due to lack of space. The two proofs share quite some similarities despite substantially different final results. Three ideas guide our proofs: First, we prove  that it is possible to localize $M$ on a certain class of open sets which behaves nicely with the manifold structure, the strongly convex sets which we denote as $\O$. This is closely related to the notion of pre-sheaf~\cite{de2012manifold}. Secondly, we characterize $M$ on small open-sets. In the scalar case, we will study the representation of locally constant functions. In the vector case, we will show that locally, the image $M(1_Uc)$ of a vector field $c$ is co-linear to $c$ provided that $U$ is small enough. We will also show that those local properties are independent of the position on the manifold $\M$ via a connectedness argument.  Thirdly and finally, we combine a compacity and a density argument to extend this characterization to $\M$, which is developed in Sec. \ref{proof-conclusion}. Throughout the presentation, we will use the following definitions and theorems obtained from other works:

\begin{definition}[Strong convexity, from \cite{gudmundsson2004introduction}]
Let $\O$ be the collection of open sets which are bounded and strongly convex, i.e. such that any points $p,q$ in such a set can be joined by a geodesic contained in the set. Furthermore let $\Op=\{A\in \O: \, \exists B\in \O, \bar A\subset B\text{ and }\omega(\bar A\backslash A)=0\}$.
\end{definition}

The intuition behind the definition of $\Op$ is that all of its elements are contained in a `security' open set,which avoids degenerated effects on the manifold. In particular, this allows to control the boundary of a given open set.

\begin{theorem}[theorem adapted from \cite{gonnord1998calcul,gudmundsson2004introduction}]
\textbf{(1)} $\Op$ is a system of neighborhoods. \textbf{(2)} Any element of $\O$ is diffeomorph to $\mathbb{R}^d$. \textbf{(3)} Both $\O $ and $\Op$ are stable by intersection.
\end{theorem}

\begin{theorem}[Flowbox theorem, as stated in \cite{calcaterra2008lipschitz}]\label{flowbox}Let $f,g\in \CcV$. For any $m\in \M$ with $f(m)\neq 0$ and $g(m)\neq 0$, there exists an open set $U \subset \M$ and $\phi\in \DiffM$ such that $\phi(m)=m$ and $L_\phi(1_Uf)=1_{\phi(U)}g$.
\end{theorem}


We will now present some lemmata that are necessary for the proofs of theorems \ref{main-thm-scalar} and \ref{main-thm-vector}. As a first step, we argue that one may assume $M(0)=0$ where $0$ denotes the constant $0$-function. This is because in the appendix we show that $M(0)$ is a constant function $C$, with $C = 0$ if $\omega(\M)=\infty$. Therefore, we may substract $C$ from $\rho$ and $\lambda$, leaving us with having to show the theorems only for $M(0) = 0$. \\
Next, a key idea of the proof is to exploit the flexibility of the deformation equivariance to \emph{localise} the input, i.e. to show that the image of compactly supported functions is also compactly supported.  To do so, the following lemma provides a way of collapsing an open ball into a singleton while maintaining a good control on the support of the diffeomorphism.
\begin{lemma}[Key lemma]\label{key}
Let $\epsilon>0$. There exists a sequence of diffeomorphisms $\phi_n:\mathbb{R}^d\to \mathbb{R}^d$, compactly supported in $\mathcal{B}(0,1+\epsilon)$ such that:
$$\phi_n(\mathcal{B}(0,1))=\mathcal{B}(0,\frac 1n)\,,$$
and
$$
\sup_{u\in \mathcal{B}(0,1)}\Vert d\phi_n(u)\Vert\leq \frac 1n~.$$
\end{lemma}
\begin{proof}
Set $\phi_n(u)=f_n(\Vert u\Vert)u$, where $$f_n(r)=\begin{cases}\frac 1n&\text{, if }|r|\leq 1\\
1&\text{, if }|r|\geq 1+\epsilon\,,
\end{cases}
$$
and $f_n$ is smoothly interpolated for $|r|\in [1,1+\epsilon]$ in a way that it remains nondecreasing. It is then clear that $\phi_n$ fulfills the desired properties.
\end{proof}

We will often use that if the support of $\phi\in\DiffM$ is such that $\text{supp}(\phi)\cap U=\emptyset$, then for any $f\in\Lp$ one has $1_Uf=L_\phi(1_Uf)$. This implies the following important lemma, for which a rigorous proof can be found in the appendix:
\begin{lemma}\label{presheafO1}
Let $U\in \Op$ and $M$ as in Theorem \ref{main-thm-scalar} or Theorem \ref{main-thm-vector}. Then, for any $f\in E$, where $E=\Lp$ or $E=\LpV$ respectively, we have:
$$M[f1_U]=1_UM[f]\,.$$
Furthermore, if $U$ is any closed set, the same conclusion applies.
\end{lemma}

Equipped with this result, our proof will characterize the image of functions of the type $c1_U$ where either $c\in \mathbb{R}$, or $c$ is a vector field which can be straightened (isomorphic to a constant vector), via the following Lemma. In the Vector case:

\begin{lemma}[Image of localized vector field]\label{vector-small}
For $M$ as in Theorem  \ref{main-thm-vector} there is  $U\in \Op $, and $\lambda(U)$ such that for any $f\in L^p_\omega(M,TM)$:
\begin{equation}
M[f1_U]= 1_U\lambda(U) f\,.
\end{equation}
\end{lemma}
 Here is the scalar case:
 
 \begin{lemma}[Image of constant functions, scalar case]\label{scalar-small}Let $M$ as in Theorem \ref{main-thm-scalar}. For any $U\in \Op$ and $c\in \mathbb{R}$,  then: $M(c1_U)=h(c,U)1_U$. Furthermore, $c\to h(c,U)$ is Lipschitz for any $U\in \Op$.
\end{lemma}
 
 At this stage, we note that both representations are point-wise, and the next steps of the proofs will be identical both for the scalar and vector cases. The extension to $\Lp$ or $\LpV$ will be done thanks to:
 \begin{lemma}[Image of a disjoint union of opensets]\label{disjoint}Let $U_1,...,U_n\in \O$ and $M$ as in Theorem \ref{main-thm-vector} or Theorem \ref{main-thm-scalar}, s.t. $\forall i\neq j, \overline{U_i}\cap \overline{U_j}=\emptyset$. Then for any $f\in \LpV$:
\[M[\sum_{i=1}^n1_{U_i}f]=\sum_{i=1}^nM[1_{U_i}f]\,.\]
\end{lemma}
This lemma states that we can completely characterize $M$ on disjoint union of simple sets. We will then need an argument similar to Vitali covering Lemma  in order to "glue" those open sets together, which shows that simple functions with disjoint support can approximate any elements of $\Lp$ or $\LpV$ (we only state the lemma for $\Lp$ as our proof on $\LpV$ does not necessarily need this result):
 \begin{lemma}[Local Vitali]\label{vitali}For $f\in \Cc$ and $m\in\M$, there exists  $U\in \Op$ with $m\in U$, such that for any $\epsilon>0$, there exist subsets $U_1,...,U_n\in\Op$ with $U_i\subset U$ and numbers $c_1,...,c_n\in \mathbb{R}$ such that:
$$\Vert \sum_n 1_{U_n}c_n-1_Uf\Vert <\epsilon\,.$$
\end{lemma}
 Note that this type of covering is not possible on any open set without further assumptions on the manifold, such as bounds on its Ricci curvature \cite{lott2009ricci}. Fortunately, we will only need a local version which is true because charts are locally bi-Lipschitz. Both Lemma \ref{disjoint} and Lemma \ref{vitali} imply that:

 \begin{proposition}\label{extension}
Consider $M$ from either Theorem \ref{main-thm-scalar} or \ref{main-thm-vector}. Assume that there exists $U\in \Op$ such that $M(c1_V)=h(c,V)1_V$ for any $V\subset U$, with $V\in\Op$, where  $c$ is either a  vector field in the case $E=\LpV$ or a constant scalar in the case $E=\Lp$. If we further assume that $c\to h(c,U)$ is $L$-Lipschitz, then
$$\forall f\in E,  \forall m\in \M, M[1_Uf](m)=1_Uh(f(m),U)\,.$$
Furthermore, it does not depend on $U$, meaning that for any other such $\tilde U$, we have:
$$\forall f\in E,  \forall m\in U\cap \tilde{U}, M[1_{\tilde U}f](m)=1_{ U} h(f(m),U)\,.$$
\end{proposition}

 We briefly discuss the intuition behind Theorem \ref{main-thm-vector}. It is linked to the idea that the operators $M$ at hand have to commute with local rotations, and this even for locally constant vector fields. We reduce the characterisation of deformation-equivariant vector operators using an invariance to  symmetry argument:  functions which are invariant to rotations are multiples of a scalar. The intuition is contained in the following lemma, which is commonly used in physics:
\begin{lemma}[Invariance to rotation]\label{invariance-rotation}
Let $f:\mathbb{R}^d\to \mathbb{R}^d$ such that for any $W\in O_d(\mathbb{R})$ and $x\in \mathbb{R}^d$, one has $f(Wx)= W f(x)$. Then, there is $\lambda:\mathbb{R}^d\to \mathbb{R}, f(x)=\lambda(\Vert x\Vert) x$.
\end{lemma}
\begin{proof}
We write $f(x)=\lambda(x) x +x^{\perp}$, with $x^{\perp}(m)\neq0$ and $x^{\perp}\perp x$. Then, we introduce $W\in O_d(\mathbb{R})$ such that $Wx^{\perp}(m)=-x^{\perp}(m)$ and $Wx(m)=x(m)$. From $f(x)=f(Wx)=Wf(x)$ we deduce   that $x^{\perp}=0$. Next, $\lambda(Wx)=\lambda(x)$ thus $\lambda(x)=\lambda(x')$ for any $\Vert x\Vert=\Vert x'\Vert$.
\end{proof}

\paragraph{Distinction between scalar and vector case} The scalar case is simpler to handle than the vector case: there are several more steps for the proof of Theorem \ref{main-thm-vector}, one needs to show that the point-wise non-linearity is actually a scalar multiplication. We also highlight that the non-linearity is fully defined by its image on locally constant functions.


Finally, we conclude the proof of the theorem by appealing to a common density argument of the functions smooth with compact support, combing all the lemmata we have just presented in Sec. \ref{proof-conclusion}.

\subsection{Proofs conclusions (common to the scalar and vector case)}\label{proof-conclusion}
In this section, we prove that the local properties of $M$ can be extended globally on $\M$. The main idea is to exploit the well-known Poincaré's formula, which states that:
$$1_{\cup_i U_i}=\sum_{k=1}^n(-1)^k\sum_{i_1<...<i_k}1_{U_{i_1}\cap U_{i_2}\cap ...\cap U_{i_k}}\,,$$
and to localize the action of $M$ on each $U_{i_1}\cap U_{i_2}\cap ...\cap U_{i_k}\in \Op$  thanks to Lemma \ref{presheafO1}.

\begin{proof}[Proof of Theorem \ref{main-thm-scalar} and Theorem \ref{main-thm-vector}]
Let $f$ be a smooth and compactly supported function. Further consider $\cup_{i\leq n}U_i$ a finite covering of its support with $U_i\in \Op$. Using an inclusion-exclusion formula together with Lemma \ref{presheafO1}, we obtain
\begin{align*}
1_{\cup_i U_i}M[f]&=\sum_{k=1}^n(-1)^k\sum_{i_1<...<i_k}1_{U_{i_1}\cap U_{i_2}\cap ...\cap U_{i_k}}M[f]\\
&=\sum_{k=1}^n(-1)^k\sum_{i_1<...<i_k}M[f 1_{U_{i_1}\cap U_{i_2}\cap ...\cap U_{i_k}}]\,,
\end{align*}
where we used that $U_{i_1}\cap U_{i_2}\cap ...\cap U_{i_k}\in \Op$.  Now, the support of $f$ is closed and included in $\cup_iU_i$. Thus using Lemma \ref{presheafO1}:
$$M[f]=\sum_{k=1}^n(-1)^k\sum_{i_1<...<i_k}M[f 1_{U_{i_1}\cap U_{i_2}\cap ...\cap U_{i_k}}],$$

Note that if $\rho$ is a pointwise operator with $\rho(0)=0$, then $\rho(1_Uf)=1_U\rho(f)$ and

\begin{align}
 \sum_{k=1}^n(-1)^k\sum_{i_1<...<i_k}\rho(f 1_{U_{i_1}\cap U_{i_2}\cap ...\cap U_{i_k}})&=\sum_{k=1}^n(-1)^k\sum_{i_1<...<i_k} 1_{U_{i_1}\cap U_{i_2}\cap ...\cap U_{i_k}}\rho(f)\\
 &=1_{\cup_i U_i}\rho(f)=\rho(f)\,.
 \end{align}
 Thus, $Mf=\rho(f)$ where $\rho$ is obtained from Lemma \ref{vector-small} or \ref{scalar-small} combined with Prop \ref{extension}. We conclude by density in $\Lp$ or $\LpV$ respectively. This ends the proof.
\end{proof}

\section{Remarks and conclusion}
In this work, we have fully characterized non-linear operators which commute under the action of smooth deformations. In some sense, it settles the intuitive fact that commutation with the whole diffeomorphism group is too strong a property, leading to a small, nearly trivial family of \emph{non-linear} intrinsic operators. While on their own they have limited interest for geometric deep representation learning, they can `upgrade' any family of linear operators associated with any group $G \subset \DiffM$ into a powerful non-linear class --- the so-called GDL Blueprint in \cite{bronstein2021geometric}. Also, this result is a first step towards characterizing the non-linear operators which commute with Gauge transformations and could give useful insights for specifying novel Gauge invariant architectures. We now state a couple of unsolved questions and future work.

\paragraph{On the commutativity assumption:}  {\color{black} Several examples and approximation results \cite{kumagai2020universal}\cite{yarotsky2022universal} exist for operators that commute with Lie groups and discrete groups \cite{keriven2019universal}. In this case, it is possible to define a measure on the group that is invariant by the group action (called the Haar measure), which makes it possible to define convolutions. Roughly, non-linear operators covariant with some actions of those groups can be thought of as an approximation by a Group Convolution Neural Networks. It is important to note that the inputs of the operators described in these articles are functions that take real values; the much more general class of inputs that take values in vector bundles is, to our knowledge, not covered in the literature. To our knowledge, we are the first work to study the design of equivariant Neural Networks that process vector fields defined over a manifold. In this setting even for} $\M=\mathbb{R}^d$, it is unclear which type of non-linear operators commute with smaller groups of symmetry such as the Euclidean group. In fact, a generic question holds for manifolds: for a given symmetry group $G$, what is elementary non-linear building block of a Neural Network? This could be, for instance, useful to design Neural Networks which are Gauge invariant. It is an open question for future work which would be relevant many applications in physics~\cite{eickenberg2022wavelet}. {\color{black} Furthermore, the fact that the characterization of diffeomorphism invariant operators we exhibited in this paper is very restrictive opens the way for the study of other non-locally 'smaller' compact groups; we believe that any results in that direction are completely novel.}

\paragraph{Example of vector operators for $L^\infty$} It is slightly unclear how the vector case $p=\infty$ can be handled in our framework, yet \cite{baez1994diffeomorphism} seems to have interesting insights toward this direction. 

\paragraph{Linearization of $\DiffM$} In this work, we considered an exact commutation between operators and a symmetries: however, it is unclear which operators approximatively commute with a given symmetry group. Such operators would be better to linearize a high-dimensional symmetry group like $\DiffM$. An important instance of non-linear operators that are non-local and that `nearly' commute with diffeomorphisms is the Wavelet Scattering representation \cite{mallat2012group, bruna2013invariant,oyallon2017analyzing}.

\begin{ack}
EO was supported by the Project ANR-21-CE23-0030 ADONIS and EMERG-ADONIS from Alliance SU. GSP was also supported by France Relance and Median Technologies; he would like to thank very much NeurIPS Foundation for their financial support (NeurIPS 2022 Scholar Award).
\end{ack}

\bibliographystyle{plain}
\bibliography{references}

\section*{Checklist}

\begin{enumerate}

\item For all authors...
\begin{enumerate}
  \item Do the main claims made in the abstract and introduction accurately reflect the paper's contributions and scope?
    \answerYes{}
  \item Did you describe the limitations of your work?
    \answerYes{}
  \item Did you discuss any potential negative societal impacts of your work?
     \answerNA{}
  \item Have you read the ethics review guidelines and ensured that your paper conforms to them?
    \answerYes{}
\end{enumerate}

\item If you are including theoretical results...
\begin{enumerate}
  \item Did you state the full set of assumptions of all theoretical results?
    \answerYes{}
        \item Did you include complete proofs of all theoretical results?
    \answerYes{}
\end{enumerate}

\item If you ran experiments...
\begin{enumerate}
  \item Did you include the code, data, and instructions needed to reproduce the main experimental results (either in the supplemental material or as a URL)?
    \answerNA{}
  \item Did you specify all the training details (e.g., data splits, hyperparameters, how they were chosen)?
    \answerNA{}
        \item Did you report error bars (e.g., with respect to the random seed after running experiments multiple times)?
    \answerNA{}
        \item Did you include the total amount of compute and the type of resources used (e.g., type of GPUs, internal cluster, or cloud provider)?
    \answerNA{}
\end{enumerate}

\item If you are using existing assets (e.g., code, data, models) or curating/releasing new assets...
\begin{enumerate}
  \item If your work uses existing assets, did you cite the creators?
    \answerNA{}
  \item Did you mention the license of the assets?
    \answerNA{}
  \item Did you include any new assets either in the supplemental material or as a URL?
    \answerNA{}
  \item Did you discuss whether and how consent was obtained from people whose data you're using/curating?
    \answerNA{}
  \item Did you discuss whether the data you are using/curating contains personally identifiable information or offensive content?
    \answerNA{}
\end{enumerate}

\item If you used crowdsourcing or conducted research with human subjects...
\begin{enumerate}
  \item Did you include the full text of instructions given to participants and screenshots, if applicable?
    \answerNA{}
  \item Did you describe any potential participant risks, with links to Institutional Review Board (IRB) approvals, if applicable?
    \answerNA{}
  \item Did you include the estimated hourly wage paid to participants and the total amount spent on participant compensation?
    \answerNA{}
\end{enumerate}

\end{enumerate}

\newpage
\appendix
\section{Technical Lemmata}\label{appendix}
\begin{proof}[Proof of Lemma \ref{compact-support}]
We simply exihibit the proof for $E=L^2_\omega(\mathcal M,T\mathcal M)$. Indeed, let $f\in L^2_{\omega}(\mathcal{M},T\mathcal{M})$, then:
\begin{align}
\Vert L_\phi f\Vert^2&=\int g(L_\phi f,L_\phi f)d\omega\\
&=\int_{\text{supp}(\phi)}g(L_\phi f,L_\phi f)d\omega+\int_{\mathcal{M}\backslash\text{supp}(\phi)}g(L_\phi f,L_\phi f)d\omega\\
&=\int_{\phi(\text{supp}(\phi))}g(d\phi^{-1}.f,d\phi^{-1}.f)\det(J\phi^{-1})d\omega'+\int_{\mathcal{M}\backslash\text{supp}(\phi)}g(f,f)d\omega\\
&\leq \int_{\text{supp}(\phi)}g(f,f)\Vert d\phi^{-1}\Vert^2\det(J\phi^{-1})d\omega'+\Vert f\Vert^2\\
&\leq (\sup_{\omega \in \text{supp}(\phi)} \Vert d\phi^{-1}(\omega)\Vert^{2(d+1)}+1)\Vert f\Vert^2<\infty\\
\end{align}

Thus, $L_\phi$ is bounded.
\end{proof}

\subsection{A remark on the Flowbox theorem}\label{remark-flowblox}
Usually, the Flowbox Theorem (here Theorem \ref{flowbox}) is stated for a (often local) diffeomorphism. If $c(m)\neq 0, \tilde c(m)\neq 0$, then there exists $U,V$ and $\phi:U\to V$ a diffeomorphism such that $m\in U\cap V$ and $L_{\phi}(1_Uc)=1_V\tilde c$. However, we note that thanks to Theorem 4 of \cite{palais1960extending}, it is possible to find $\tilde U$  smaller such that there exists $\tilde \phi:\M\to\M$ which is a global diffeomorphism and $\forall m\in \tilde U, \tilde \phi(m)=\phi(m)$. In this case, $\tilde \phi, \tilde U$ and $\tilde V=\tilde \phi(\tilde U)$ are the candidates of our statement in Theorem \ref{flowbox}.   As this is quite technical and rather intuitive, we skipped this remark in the main paper.
\subsection{Spatial localization (common to the scalar and vector case)}\label{localisation}


We now explain how to localize our operator $M$. Equipped with Lemma \ref{key}, we can extend our contraction result on $\mathbb{R}^d$ to $\M$ as follow:

\begin{corollary}[Contraction of an openset]\label{contract}
For any $U\in \O$ and $W$ openset such that $\bar U\subset W\subset \mathcal{M}$, there exists $\phi_n$ supported on $W$ such that for any $f\in \LpV$:
$$L_{\phi_n}(1_Uf)\to 0\,.$$
\end{corollary}
\begin{proof}
We prove first the result for $U=\mathcal{B}(0,1)$ and $\bar U\subset W$. In this case, it is possible to find $\epsilon>0$ such that $\mathcal{B}(0,1+\epsilon)\subset W$. Now, taking $\phi_n^{-1}$ as in Lemma \ref{key}, we get:
\begin{align}
\int_{\mathbb{R}^d} \Vert L_{\phi_n}(1_\mathcal{B}(0,1)f)(u)\Vert^p\,du&=\int_{\mathbb{R}^d} \Vert 1_\mathcal{B}(0,1)(\phi_n^{-1}(u))d\phi_n(u).f(\phi_n^{-1}(u))\Vert^p\,du\\
&=\int_{\mathbb{R}^d}\Vert 1_{\mathcal{B}(0,\frac1n)}(u)d\phi_n(u).f(nu)\Vert^p\,du\\
&=\frac{1}{n^d}\int_{\mathbb{R}^d}1_{\mathcal{B}(0,1)}(u)\Vert d\phi_n(\frac un).f(u)\Vert^p \,du\\
&\leq \frac{1}{n^{d+1}}\Vert 1_{\mathcal{B}(0,1)} f\Vert^p\to 0
\end{align}
Next, getting back to the manifold, we know that if $U\in \Op$, there is $V\in \O$ such that $\bar U\subset V$. We can thus find an openset $\mathcal{B}\subset V$, such that in the chart of $V$, $\mathcal{B}$ is an open ball, and $U\subset \mathcal{B}\subset W$. We can thus apply the technique derived above to get $\phi_n:V\to V$, compactly supported, which contracts $\mathcal{B}$(and thus $U$) to 0 and supported in $W$. Since it is smooth, compactly supported on $W$, we can extend it on $\M$ and we get the result.
\end{proof}

Next, this technique can be used to build a sequence of contraction, which allows to explicitly localize the image of a compactly supported function, as follow:

\begin{lemma}[ Lemma \ref{presheafO1} restated for closed sets]\label{presheaf}
Let $F\subset \mathcal{M}$ a closed set. Then, for any $f\in \Lp$, we have:
$$M[f1_F]=1_FM[f]$$
\end{lemma}
\begin{proof}
Because $\mathcal{M}$ is a manifold, it is second countable and thus there is a countable collection of opens such that $\mathcal{M}\backslash F=\cup_{i\geq 0}U_i$ with $U_i\in \O$. We use Lemma \ref{union-ball} and, we apply the dominated convergence theorem to $f_n=1_{\cup_{i\leq n}U_i}f$ to conclude.
\end{proof}
\begin{proof}[Proof of Lemma \ref{presheafO1}]
We note that if $U\in \Op$, then $\omega(\bar U\backslash U)=0$ and we can thus use the Corollary \ref{contract}  to conclude.
\end{proof}

\subsection{Action on locally constant functions, for the scalar and vector cases}
We now prove the part specific to the vector field setting, i.e., that the action of $M$ is locally a multiplication by a scalar.

\begin{proof}[Proof of Lemma \ref{vector-small}]
\textbf{Step 1:}$M(1_Uc)(m)=1_V\lambda(m,U,c)c$ such that $c(m)\neq 0, \forall m\in U$.

Let $c\in C_c^\infty(\mathcal M, T\mathcal M)$. For $U\in \Op$, $m_0\in U$, fix a chart $\psi:U\to \mathbb{R}^d$,  $\psi(m_0)=0$ and $c$ is  constant in $\psi$ denoted $c^\psi\in \mathbb{R}^d$, which is possible thanks to the Theorem \ref{flowbox}. This can also be written as for $m$ in a neighborhood of $m_0$: $$d\psi(m).c(m)=c^\psi\,.$$ Following the strategy in Lemma \ref{invariance-rotation},there is $W\in \mathcal O_d$ such that $Wc^\psi= c^\psi$ and $Wv= -v$ for any vector $v$ orthogonal to $c^\psi$.   By compacity, we can find $A$ an open set small enough, with boundary of measure 0,  such that $0\in A$, and $\mathcal{W}A\subset \psi(U)$ for any $\mathcal{W}\in \mathcal{O}_d$. Now, setting $\tilde\phi=\psi^{-1}\circ W\circ \psi$, which is well defined on the open $\cup_{\mathcal{W}\in \mathcal{O}_d}\mathcal{W} A$, using Theorem 4 of \cite{palais1960extending}(see remark Sec. \ref{remark-flowblox} of the appendix), we can can extend $\phi$  globally such that on a local neighborhood, $\forall m\in \tilde U, \phi(m)=\tilde \phi(m)$. Now, up to taking $A$ even smaller, we can use: $V=\overline{\psi^{-1}(\cup_{n\in \mathbb{Z}}W^nA)}\subset U$, which is closed with a measure 0 boundary(we have a countable union).  We get:
\begin{align}
    L_{\phi}(1_{V}c)(m_0)&=[d\psi^{-1}(m_0)\circ W\circ d\psi(m_0)]c( m_0)1_{V}\\
    &=1_{V}c(m_0)\,.
\end{align}

Let us denote $p_{c^\psi}^\perp$ the orthogonal projection (with respect to the Euclidean scalar product) on the orthogonal plane to $c^\psi$.

As $V\subset U$, $V$ is closed and $U\in \Op$ from Lemma \ref{presheaf}, we know that:
$$M(c)(m_0)=M(1_U c)(m_0)=M(1_V c)(m_0)=\lambda(m_0,c,U)d\psi^{-1}(0)c^\psi +d\psi^{-1}(0)p_{c^\psi}^\perp M(1_Vc)(m_0)$$
Yet, on the other hand:
\begin{align}
    L_{\phi}M(1_Vc)(m_0)&=\lambda(m_0,c,U)d\psi^{-1}(0)c_\psi -d\psi^{-1}(0)p_{c^\psi}^\perp M(1_Vc)(m_0)
\end{align}

As this is true for any $m_0$, we thus proved that:
$$M(1_U c)(m)=1_U\lambda(m,U,c)c$$

\textbf{Step 2:} In fact, $\lambda(m,c,U)=\lambda(m,U)$ if $c$ does not cancel on $U$ and $m\in U$.

Let $c,\tilde c$ be two vector fields as above and defined on $U$ both not equal to 0, and $m\in U$. Using the Theorem \ref{flowbox} combined with the remark of Sec. \ref{remark-flowblox} of the appendix, there exists $\phi:\M\to \M$ a diffeomorphism and $\tilde V, V\subset U$ and $m\in \tilde V\cap V$, such that $L_\phi (1_{ V}c(m))=1_{\tilde V}\tilde c(m)$ and $\phi(m)=m$ 
Now, we could take a smaller closed set $V\subset U$ with measure 0 boundary, so that $M[1_{V}c](m)=M[1_{U}c](m)=M[c](m)$, which would lead to, following a similar argument to above:
$$\lambda(m,\tilde c,U)\tilde c(m)=M[1_{\tilde V}\tilde c(m)]=L_{\phi}M[1_{ V} c](m)=L_{\phi}(\lambda(.,c,U)c)(m)=\lambda(m,c,U)\tilde c(m)$$
and then locally $\lambda$ is independent of the choice of a vector field, which implies the desired property.  

\textbf{Step 3:} In fact, $\lambda(m,U)=\lambda(U)$.
Indeed, let $m,m_0\in V$ and $\phi\in\DiffM$ such that $\phi(m)=m_0$ (as $V$ is connex, by using Lemma \ref{permute}). Now, along the same line as above:
$$\lambda(m,U)=\lambda(m_0,U)$$

The previous results hold when the vector field can be locally straightened, however the vector fields that take value $0$ on some points of $U$ can not be straightened. We will now show that vector fields that can be straightened on $U\in \Op$ are dense dense in $C^\infty(U, T  U)$ for the $L^p_\omega$ norm. Let $f\in C^\infty(U, T U)$, let $A=\{x\in U\vert f(x)=0\}$, and $A^\epsilon=\{x\in U\vert \Vert f(x)\Vert \leq \epsilon\}$ for $\epsilon >0$. By Urysohn's lemma there is $\chi^\epsilon: U\to \mathbb R$ be such that $\chi|_{A^\epsilon}=1$ and $\chi|_{U\setminus A^{2\epsilon}}=0$. Let, 

$$f^\epsilon= f + 2\epsilon \chi^\epsilon $$

For any $x\in U$,

\begin{equation}
\Vert f^\epsilon(x) \Vert \geq \vert \Vert f(x)\Vert - 2\epsilon \chi^\epsilon(x) \vert 
\end{equation}

and by construction $\vert \Vert f(x)\Vert - 2\epsilon \chi^\epsilon(x) \vert >0$.

Therefore, 

\begin{equation}
M[f^\epsilon 1_U]= \lambda(U) f^\epsilon
\end{equation}

Furthermore for all $0<\epsilon \leq 1$, $\Vert f^\epsilon\Vert$ is bounded by $\Vert f \Vert +2$ that is integrable, so by dominated convergence theorem, $f^\epsilon \overset{L^p_\omega}{\underset{\epsilon \to 0}{\longrightarrow}} f$. So, $M[f 1_U]= \lambda(U) f$.

To end the proof, one remarks that $C_c^\infty(\mathcal M, T\mathcal M)$ is dense in $L^p_\omega(\mathcal M,T\mathcal M)$.

\end{proof}

\if False
\begin{lemma}\label{scalar}
Let $m_0\in\M$. There exists a local parametrization $\psi:U\to \mathbb{R}^d$ with $m_0\in U$ such that for any $V\subset U$, with $U\in \Op,V\in \Op$, and for any $c\in \LpV$ such that $\psi(c)$ is a constant vector, then there exists a scalar $\lambda(V)$ such that:
$$M(1_Vc)=1_V\lambda(V) c$$

Let $m_0\in M$, let $m_0\in U\in \Op$ and $c\in \LpV$. Assume that there is a chart $\psi :U\hookrightarrow \mathcal B(0,1)$ such that $c^\psi$ is constant. Then there is a neighborhood $m_0\in U_1$ of $m_0$ included in $U$ and a function $\lambda: \mathcal M\to \mathbb R$, such that, 

\begin{equation}
M[1_{U_1} c]= \lambda c 1_{U_1}
\end{equation}
\end{lemma}

\begin{proof}

Let us first remark that for bounded open $\mathcal O\subseteq  GL_n(\mathbb R)$, there is an an open $U_1^\psi\subseteq \mathcal B (0,1)$ such that , 

\begin{equation}
\forall \phi \in \mathcal O, \quad \phi(U_1)\subseteq \mathcal B(0,1)
\end{equation}

Such isomorphisms of $\mathbb R^n$ can be extended to global diffeomorphism of $M$, by applying the extension theorem of \cite{Palais}, as $\psi^{-1}:U_1^\psi\hookrightarrow M$ and $\psi^{-1}: \mathcal B(0,1)\hookrigharrow M$. 

Indeed, for the standard operator norm on $M_n(\mathbb R)$, if $\mathcal O$ is bounded there is $\mathcal B(0,r)$ such that $\mathcal O \subseteq \mathcal B(0,r)$; let $U_1^\psi= \mathbb B(0,\frac{1}{r})$  then, for any $v\in U_1^\psi$, $\Vert \phi v\Vert \leq r \Vert v\Vert\leq 1$.\\

As the norms induced by $g$ for all $m\in \mathcal M$ are uniformly equivalent to the euclidean norm on $\mathbb R^n$ the previous result replacing the standard operator norm on $\mathbb{R}^n$ by the one defined by $g_m$ on each fibers.

\textbf{Step 1:}$M(1_Vc)=1_V\lambda(c,V)c$

Fix a chart $\psi:U\to \mathbb{R}^d$,$U\in \Op$ where $c$  is the constant vector field in $\psi$ denoted $c^\psi\in \mathbb{R}^d$, which is possible thanks to the Theorem \ref{flowbox}. This can also be written as $1_V(\psi^{-1}(x))c(\psi^{-1}(x))=1_V(\psi^{-1}(x))c^\psi$ for $x\in \mathbb{B}_{\mathbb{R}^d}(0,1)$. Its support is $V$ which becomes $V^\psi= \mathcal B(0,1)$ in chart $\psi$. For any vector $v$ orthogonal to $c^\psi$, there is $W\in \mathcal O$ such that $Wc^\psi= c^\psi$ and $Wv= -rv$ with $r>0$. We fix $\bar m\in V$. Then, we can globally extend $\phi(m)=\psi^{-1}(\psi(\bar m)+W(\psi(m)-\psi(\bar m))$ around $\bar m$. It's clear that $d\phi$ is conjugated to $W$ and thus, for $\bar V\subset V$, an invariant neighborhood by $\phi$:
\begin{align}
    L_{\phi}(1_Vc)(\bar m)&=1_V(\bar m)[d\psi(\bar m)\circ W\circ d\psi^{-1}(\bar m)]c(\bar m)
\end{align}
We can thus design $W$ which cancels all other component than $c^\psi$.

Therefore,

$$M[L_\phi[1_{V^\psi}c^\psi]](m)=L_\phi M[1_{V^\psi}c^\psi](m)$$

and:

$$\langle L_{\phi}M[1_{V^\psi}c^\psi](m),v\rangle=-r\langle M[1_{\tilde U}c](m),c^\perp\rangle=\langle M[1_{\tilde U}c](m),c^\perp\rangle$$

Therefore, 

$$M[1_{V^\psi}c^\psi](m)=1_{V^\psi}(m)\lambda^\psi(m,c^\psi) c^\psi$$

Let us show that the previous equation holds in any chart as both expressions are covariant: in an other chart the left hand side becomes $M[1_{V^\psi_1}c^\psi_1](m)$ and the right hand side

\begin{equation}
\lambda^\psi(\psi_1\psi^{-1}m)d\psi_1\psi^{-1} c^\psi(\psi_1\psi^{-1}m)=  \lambda^{\psi_1}(m) c^{\psi_1}(m)
\end{equation}

In other words we have shown that when the vector field is trivializable there is a function $\lambda: \mathcal M \to \mathbb R$ such that , 

\begin{equation}
M[c]= \lambda(c) c
\end{equation}

Then for any diffeomorphism 
$\phi$, $L_\phi c$ is also trivializable which implies that

\begin{equation}
M[L_\phi c]= \lambda(L_\phi, m) L_\phi c = \lambda(c,\phi^{-1}(m)) L_\phi c
\end{equation}

Therefore, 

\begin{equation}
\lambda(L_\phi c, m)=\lambda(c,\phi^{-1}(m))
\end{equation}

The $W\in \mathcal{O}\subseteq GL_n(\mathbb R^n)$ that leave $c$ invariant in the trivializing chart of $c$ have the particularity to have as obits all vectors colinear to $c$ and we denote the complementary of the line generated by $<c>$ in $U_1$ as $U_1 \setminus <c>$; this big orbit $U_1 \setminus <c>$ is in particular dense in $U_1$ Therefore for any $m,m^{'} \in V \setminus <c>$ 

\begin{equation}
\lambda(c,m)=\lambda(c,m') 
\end{equation}

( as $m^{'}= \phi^{-1}(m)$ for some $\phi$). Therefore $\lambda(c)$ is a constant almost surely.

\textbf{Step 2:} In fact, $\lambda(f,V)=\tilde\lambda(V)$.

Let $c,\tilde c$ two vector fields as above and defined on $V$ both not equal to 0, by the Theorem \ref{flowbox}, there exists $\phi:U\to U$, $\tilde V,\bar V\subset \mathcal{M}$ a local diffeomorphism such that $L_\phi (1_{\bar V}c(m))=1_{\tilde V}\tilde c(m)$ and $\phi(m)=m$ for some $m\in V$. Using Corollary \ref{extension-M}, we can extend $\phi$ into a global diffeomorphism, and consequently:
$$\lambda(\tilde c,V)\tilde c(m)=M[1_{\tilde V}\tilde c(m)]=L_{\phi}M[1_{\bar V} c(m)]=L_{\phi}(\lambda(c,V)c)(m)=\lambda(c,V)\tilde c(m)$$
and we get the result.
\end{proof}
\fi 
The next Lemma shows that, in the scalar case, we can consider $\tilde Mf\triangleq Mf-M(0)$ for $f\in \Lp$ without losing in generality.

\begin{lemma}
Under the assumptions of Theorem \ref{main-thm-scalar}, $M(0)$ is constant, and if $\omega(\M)=\infty$, then $M(0)=0$. 
\end{lemma}
\begin{proof}
Following the Theorem 1 of \cite{michor1994n}, for any $m,m_0\in\M$, we can find $\phi$ global diffeomorphism such that $\phi(m)=m_0$. We note that $L_\phi(0)=0$ and thus for any $m\in \M$:
$$M(0)(m)=M[L_\phi (0)]=L_\phi M(0)(m)=M(0)(m_0)$$
Thus, $M(0)$ is constant, and if $\omega(\M)=\infty$, it is necessary that $M(0)=0$.
\end{proof}

The corresponding Lemma in the scalar case is substantially simpler, as strongly convex sets are connex:

\begin{proof}[Proof of Lemma \ref{scalar-small}]
Fix $m_0\in V$, and let $m\in V$, using Lemma \ref{permute}(because $V\in\Op$ is connex, we can apply a connexity argument or the transitivity argument of Theorem 1 of \cite{michor1994n} for compactly supported diffeomorphisms), we can find $\phi$  supported in $V$ such that $\phi(m_0)=m$. Thus, $L_\phi f=f$ and $Mf(m_0)=ML_\phi f(m_0)=L_\phi Mf(m_0)=Mf(m)$. Thus, $M(c1_V)=h(c,V)1_V$. The Lipschitz aspect is inherited from the fact that $M$ is Lipschitz.
\end{proof}

\subsection{Extrapolation to any good open sets (common to the scalar and vector case)}
In this section, we use the fact that we want to prove that both scalar and vector operators correspond to point-wise non-linearity, which are locally Lipschitz due to the regularity assumptions that we used.
\begin{proof}[Proof of Proposition \ref{extension}]

\textbf{Step 1:Fix $c$, for any $m\in U$ such that $V\subset U$, then $h(c,U)(m)=h(c,V)(m)$}

Indeed, we note that for $m\in U$, where we used Lemma \ref{presheafO1}:
$$M(1_Vf)(m)=1_V(m)M(f)(m)=1_U(m)M(f)(m)=M(1_U f)(m)$$
Thus, $h(c,V)|_V=h(c,U)|_V$ for any $V\subset U$.

\textbf{Step 2: extension by density, for any $f$, $M(f1_U)=1_Uh(f,U)$ for any $f\in \Lp$}.
Using Lemma \ref{vitali}, consider $f\in\mathcal{C}^\infty_c(E)$, $f_n=\sum_n 1_{U_n}c_n$, where $c_n$ is either a constant scalar, either a vector field, with disjoint support such that $\Vert 1_Uf-1_Uf_n\Vert<\epsilon$.

We know that, from Lemma \ref{disjoint} that:

$$M(1_Uf_n)=M[\sum_n 1_{U_n}c_n]=\sum_n 1_{U_n}M[1_{U_n}c_n]=\sum_n 1_{U_n} h(c_n,U)$$
Next, we note that:

\begin{align}
\Vert M1_Uf-1_Uh(f,U)\Vert&\leq \Vert 1_U(Mf_n-Mf)\Vert+\Vert 1_UMf_n-1_U h(f_n,U)\Vert\\
&+\Vert 1_U(h(f_n,U)-h(f,U))\Vert\\
&\leq 2L\Vert 1_U(f_n-f)\Vert
\end{align}

and from this, given that $h(.,U)$ is $L$-Lipschitz, we conclude by density of $\Cc$ in $\Lp$.

\textbf{Step 3: Independence from $U$}

Step 1 allows for the following definition of a global $h$ from local $h_U$: let $m\in \mathcal M$, pose,

\begin{equation}
\forall U\in \Op \quad h(f(m)):= h(f(m),U) 
\end{equation}

In the scalar case and in the vector case, one can build a scalar function and vector function such that, $f(m)= \mu\in \mathbb R$ or $f(m)= c \in T_x M$ (as shown in Step 3 of proof of \ref{vector-small}). Therefore in the scalar case $h$ is a function from $\mathbb R$ to $\mathbb R$ and in the vector case for any $x\in \mathcal M$ and $v\in T_x\mathcal M$, $h(x)=\lambda x$.

\end{proof}

We only prove the Vitali version for $\Lp$, as the proof for $\LpV$ would be identical, replacing solely the scalar by constant vector fields in their local parametrization.

\begin{proof}[Proof of Lemma \ref{vitali}]
We consider $U$ small enough such that $U\in \Op$, $m\in U$ and $\exp_m:\mathcal{B}\to U$ is locally a diffeomorphism from $\mathcal{B}\subset T\M_m$, and let $U_i=\exp_m(\mathcal{B}_i)$ with $\mathcal{B}(x_i,r_i)\subset \mathcal{B}$, which is strongly convex and thus $U_i\in \dot{\mathcal{O}}_1$. We remind that $\exp_m$ is bi-Lipschitz on the bounded set $U$. In this case, there is $C_1,C_2>0$ such that for any $x_i,r_i$ with $\mathcal{B}(x_i,r_i)\subset \mathcal{B}$, we have $r_i^d\leq \lambda(\mathcal{B}(x_i,r_i))\leq C_1\omega(U_i)\leq  C_2\lambda(\mathcal{B}(x_i,r_i))\leq C_dr^d$. By Vitali's lemma, we have for any $\epsilon>0$ and  $r>0$, the existence of some $x_i,r_i<r$:
$$\Vert 1_\mathcal{B}-\sum_{i=1}^n1_{\mathcal{B}(x_i,r_i)}\Vert^p\leq \epsilon^p$$
For $f$ smooth, let:
\begin{align}
\Vert f(x)1_U-\sum_{i=1}^nf(x_i)1_{U_i}\Vert^p&\leq \Vert\sum_{i=1}^n (f(x)-f(x_i))1_{U_i}\Vert^p+\Vert 1_{U\backslash(\cup_i U_i)}f(x)\Vert^p
\end{align}

Now, as $\exp_m$ is  bi-Lipschitz, we get a $r$ small enough such that $|f(x)-f(x_i)|<\epsilon$. Next, because the sets are disjoint:
\begin{align}
\Vert\sum_{i=1}^n (f(x)-f(x_i))1_{U_i}\Vert^p&=\sum_{i=1}^n \int_{U_i}|f(x)-f(x_i)|^p\\
&\leq \sum_{i=1}^n \omega(U_i)\epsilon^p\\
&\leq \epsilon^p\omega(U))\,.
\end{align}
Now, using $|f(x)|\leq \Vert f\Vert_\infty$, we get:
$$\Vert 1_{U\backslash(\cup_i U_i)}f(x)\Vert^p\leq \Vert f\Vert_\infty \epsilon^p$$
And: $$\Vert f-\sum_{i=1}^n f(x_i)1_{U_i}\Vert<(1+\omega(U))^{1/p}\epsilon\,.$$
\end{proof}
The following Lemma allows to build diffeomorphism with compact support - we give this proof for the sake of completeness, at it is proved in \cite{michor1994n}.
\begin{lemma}\label{permute}
 Fix $\rho>0$, and $x_0,x_1\in \mathcal{B}(0,\rho)$, there exists $\phi$ diffeomorphism, such that $\phi(x_0)=x_1$ and $\text{supp}(\phi)\subset  \mathcal{B}(0,\rho)$.
\end{lemma}
\begin{proof}
Consider $f$, smooth, supported in $[-1,1]$ and such that $f(0)=1$. We will use a connexity argument: let us fix $x_0\in  \mathcal{B}(0,\rho)$. Let's consider $\Gamma=\{x \in  \mathcal{B}(0,\rho): \exists \phi\text{ diffeomorphism  }\phi(x)=x_0, \text{ supp}(\phi)\subset  \mathcal{B}(0,\rho)\}$. Let $x_1\in \Gamma$, then there is $\eta<\frac 12$, $\mathcal{B}(x_1,\eta)\subset \mathcal{B}(0,\rho)$. For $x_2$ such that $\Vert x_1-x_2\Vert\leq \frac{\eta}{4\sup |f'|}$, we introduce:

\[\tau(x)=(x_2-x_1)f(\frac{\Vert x-x_1\Vert^2}{\eta^2})\,.\]
We have that $\text{supp}(\mathbf{I}-\tau)\subset \mathcal{B}(x_1,\eta)$, and:
\[\frac{\partial \tau}{\partial x}(x)=2\frac{(x_2-x_1)\langle x-x_1,x_1\rangle}{\eta^2} f'(\frac{\Vert x-x_1\Vert^2}{\eta^2})\]
leading to:
\[\Vert\frac{\partial \tau}{\partial x}(x)\Vert< \frac 12\]
This implies that the spectrum of $\partial \tau$ is in $[0,1[$ and thus, $\mathbf{I}-\partial \tau$ is invertible. Now, by assumption, we know there is $\phi$ such that $\phi(x_1)=x_0$, compactly supported in $\Omega$. Introducing $\phi'=\phi\circ (\mathbf{I}-\tau)$, then $\phi'$ is a diffeomorphism, compactly supported in $\Omega$ and $\phi'(x_2)=\phi(x_1)=x$, thus $x_2\in \Gamma$. This shows $\Gamma$ is open. But also $\Gamma$ is closed (otherwiwe, we can make a path ...). Thus, by connexity $\Gamma=\Omega$.
\end{proof}

The next Lemma is crucial in our proof, and allows to characterize union of well behaving opensets:

\begin{lemma}\label{union-ball}Let $n\geq 0$, $\{U_i\}_{i\leq n}\subset \Op$ and $F$ a closed set such that $\bar U_i\cap F=\emptyset, \forall i$. Then for any $f\in \LpV$:
$$1_FM[(1_F+1_{\cup_{i\leq n}U_i})f]=1_FM[1_Ff]$$
\end{lemma}
\begin{proof}
We work by induction on $n$. For $n=0$, the result is true. Then, let's write $U_{n+1}^\epsilon=\{x,d(U_{n+1},x)<\epsilon\}$. It's an openset which contains $\bar U_{n+1}$, and by assumption we can pick $\epsilon$ small enough such that $U_{n+1}^\epsilon \cap F=\emptyset$. Next, let's apply {\color{black}Corollary} \ref{contract} to $U_{n+1}$ and $W=U_{n+1}^\epsilon$. Then:
\begin{align}
1_FM[(1_F+1_{(\cup_{i\leq n}U_i\backslash U_{n+1}^\epsilon)\cup  U_{n+1}})f]&=L_{ {\color{black}\phi_n}}1_FM[(1_F+1_{(\cup_{i\leq n}U_i\backslash U_{n+1}^\epsilon)\cup  U_{n+1}})f]\\
&=1_FM[L_{\phi_n}(1_Ff +1_{(\cup_{i\leq n}U_i\backslash U_{n+1}^\epsilon)\cup  U_{n+1}}f)]\\
&\to 1_FM[1_Ff+1_{(\cup_{i\leq n}U_i\backslash U_{n+1}^\epsilon)}f]
\end{align}
Now, we remark that:
\begin{align}
    1_FM[1_Ff+1_{(\cup_{i\leq n}U_i\backslash U_{n+1}^\epsilon)}f]=1_FM[1_Ff+1_{\cup_{i\leq n}U_i}(1_{\mathcal{M}\backslash U_{n+1}^\epsilon}f)]
\end{align}
And we apply the induction hypothesis to $(1_{\mathcal{M}\backslash U_{n+1}^\epsilon}f)$.
\end{proof}
The next Lemma is crucial in our proof, and allows to characterize disjoint union of well behaving opensets:
\begin{proof}[Proof of lemma \ref{disjoint}]
We note that $\cup_{i=1}^n\overline{U_i}=\overline{\cup_{i=1}^nU_i}$. Thus, using Lemma \ref{presheaf}, given this union is closed and disjoint and  as for any closed set $F$,
\begin{equation}
M[f 1_F]1_{F^c}= M[0]1_{F^c} =0
\end{equation}

the following linearity property holds,

$$
M[\sum_{i=1}^n1_{\bar U_i}f]=\sum_{i=1}^n1_{\bar U_i}M[f]=\sum_{i=1}^nM[1_{\bar U_i}f]$$
Now, we conclude as the boundaries have measure 0.
\end{proof}

\end{document}